\documentclass[10pt]{article} 

\usepackage[preprint]{rlj} 

\usepackage{amssymb}            
\usepackage{mathtools}          
\usepackage{mathrsfs}           
\usepackage{graphicx}           
\usepackage{subcaption}         
\usepackage[space]{grffile}     
\usepackage{url}                
\usepackage{lipsum}             

\usepackage{amsthm}
\usepackage[ruled,vlined,linesnumbered]{algorithm2e}

\usepackage{bbm,bm}

\usepackage{xspace}
\usepackage{dsfont}
\usepackage[textsize=small]{todonotes}
\usepackage{mathrsfs}
\usepackage{longtable}
\usepackage{xpatch}
\usepackage{collect}

\usepackage{cleveref}
\usepackage[compact]{titlesec}

\usepackage{booktabs}
\usepackage{multirow}
\makeatletter
\renewcommand{\todo}[2][]{\@todo[#1]{#2}}
\makeatother

\newcommand{\ip}[1]{\left\langle #1 \right\rangle}
\newcommand{\norm}[1]{\left\lVert #1 \right\rVert}
\newcommand{\R}{\mathbb{R}}

\DeclareMathOperator*{\argmax}{argmax}

\newcommand{\set}[1]{\left\{#1\right\}}
\newcommand{\E}{\mathbb E}

\newcommand{\cF}{\mathcal{F}}
\newcommand{\EE}[1]{\E\left[#1\right]}

\newcommand{\N}{\mathbb N}

\newtheorem{proposition}{Proposition}
\newtheorem{remark}{Remark}
\newtheorem{theorem}{Theorem}

\newtheorem{lemma}{Lemma}

\newtheorem{assumption}{Assumption}

\crefname{theorem}{Theorem}{Theorems}
\crefname{lemma}{Lemma}{Lemmas}
\crefname{corollary}{Corollary}{Corollaries}
\crefname{assumption}{Assumption}{Assumptions}
\crefname{proposition}{Proposition}{Propositions}
\crefname{definition}{Definition}{Definitions}
\crefname{example}{Example}{Examples}
\crefname{remark}{Remark}{Remarks}

\newcommand{\PP}{\mathbb{P}}

\newcommand{\cA}{\mathcal{A}}

\newcommand{\cE}{\mathcal{E}}

\newcommand{\cC}{\mathcal{C}}
\newcommand{\cU}{\mathcal{U}}

\newcommand{\cL}{\mathcal{L}}
\newcommand{\cD}{\mathcal{D}}
\newcommand{\cM}{\mathcal{M}}

\DeclareMathOperator{\tr}{tr}

\newcommand{\ucb}{\text{UCB}}
\newcommand{\lcb}{\text{LCB}}

\newcommand{\alg}{\text{LinOtO}\xspace}
\newcommand{\betaoff}{\beta_{\text{off}}}
\newcommand{\lcbaoff}{L_{\text{off}}}
\newcommand{\muoff}{\mu_{\text{off}}}

\title{Offline-to-Online Learning in Linear Bandits}

\setrunningtitle{Offline-to-Online Learning in Linear Bandits}

\author{Kushagra Chandak, Toshinori Kitamura, Xiaoqi Tan}

\emails{\{kchandak,toshinor,xiaoqi.tan\}@ualberta.ca}

\affiliations{
\textbf{Department of Computing Science, University of Alberta, Canada}
}

\contribution{
    This paper presents an algorithm for online learning with an additional offline dataset in the stochastic linear bandit setting with regret bounds.
    }
    {
        Prior work provides an algorithm for the multi-armed bandit case \citep{sentenac2025balancing}.
    }

\contribution{
    Empirical results showing competitiveness of our algorithm against an online and an offline solution simultaneously
    }
    {
        None
    }


\summary{
    We study online learning with an additional offline dataset in the stochastic linear bandit setting. Although this problem arises frequently in practice, the offline-to-online tradeoff remains poorly understood in structured environments. We propose a linear bandit algorithm that balances this tradeoff: it relies on offline data during early rounds, and increasingly favors exploration as the horizon grows. We establish regret bounds showing that our method is simultaneously competitive with both purely online and purely offline solutions. In particular, it achieves sublinear regret relative to the optimal action in the number of online interactions, while its regret relative to an offline reference decreases as the number of offline samples grows. Empirical results further demonstrate its effectiveness across various problem parameters.
}

\begin{document}

\makeCover  
\maketitle  

\begin{abstract}
    We study online learning with an additional offline dataset in the stochastic linear bandit setting. Although this problem arises frequently in practice, the offline-to-online tradeoff remains poorly understood in structured environments. We propose a linear bandit algorithm that balances this tradeoff: it relies on offline data during early rounds, and increasingly favors exploration as the horizon grows. We establish regret bounds showing that our method is simultaneously competitive with both purely online and purely offline solutions. In particular, it achieves sublinear regret relative to the optimal action in the number of online interactions, while its regret relative to an offline reference decreases as the number of offline samples grows. Empirical results further demonstrate its effectiveness across various problem parameters.
\end{abstract}


\section{Introduction}
\looseness=-1
We study the \emph{offline-to-online learning} problem in the stochastic linear bandit model~\citep{lattimore2020bandit}, where a learner sequentially selects actions and receives noisy linear rewards while having access to an offline dataset. Such offline-to-online learning problems arise frequently in practice, for example, a recommendation system typically has access to historical user-preference data which can be used to improve the next version of the system. Despite its practical importance, the offline-to-online learning problem remains poorly understood in structured environments, even in the simple linear bandit setting.

\looseness=-1
In offline learning, a common approach is \emph{pessimism in the face of uncertainty} \citep{li2022pessimism,xiao2021optimality}, which yields solutions that are competitive with respect to the offline data. However, such methods do not explore and may incur large regret relative to the optimal policy when the learning horizon is long. In contrast, online learning algorithms typically rely on \emph{optimism in the face of uncertainty} \citep{li2010contextual,abbasi2011improved}, which encourages exploration and achieves sublinear online regret. Nevertheless, when the learning horizon is short, optimistic strategies may explore excessively and consequently suffer large regret relative to the policy derived from the offline data. The central challenge in the offline-to-online setting is therefore to balance these two principles, effectively trading off between exploiting the offline solution and exploring to improve long-term performance \citep{sentenac2025balancing}.

\looseness=-1
We achieve this balance by managing an \emph{exploration budget} in the linear bandit setup. Intuitively, by selecting the pessimistic action computed from the offline data, the learner can ensure temporarily low regret relative to the offline dataset, thereby accumulating budget that can later be used for exploration. Once sufficient budget has been accumulated, the learner can select optimistic actions to explore potentially better alternatives. Building on this idea, in \Cref{sec:method}, we propose the algorithm \alg. In each online round, \alg carefully balances between exploration using upper confidence bounds (UCB) and budget accumulation using lower confidence bounds (LCB). \Cref{prop:loff-subopt} formally characterizes the budget accumulation by using LCB in the stochastic linear bandit model.

\looseness=-1
We show that, with high probability, the performance of our algorithm is simultaneously close to both the optimism-based online solution (\Cref{thm:linoto-reg}) and the pessimism-based offline solution (\Cref{thm:linoto-offreg}). Moreover, in \Cref{sec:experiments}, we conduct numerical experiments across a range of problem settings and demonstrate that \alg performs nearly as well as the better of the offline and online solutions. Overall, our work contributes to a growing line of research that leverages prior knowledge or hints to guide the exploration–exploitation tradeoff in online learning \citep{cutkosky2022leveraging,wei2020taking,bhaskara2023bandit,alamdari2024jump}.

\section{Related Work}
There is an extensive body of literature on both online and offline learning. In this section, we focus on three areas closely related to our offline-to-online setting.

\textbf{Warm starting online learning with offline data.}
The problem of leveraging offline data to warm start online learning has been studied across various settings. For the multi-armed bandit (MAB) case, \citet{shivaswamy2012multi} proposed the HUCB1 algorithm, which improves the regret of the standard UCB1 \citep{auer2002finite} by constructing tighter confidence intervals. \citet{cheung2024leveraging} generalized this result to scenarios with distribution shifts between the offline and the online data, establishing tight upper and lower bounds. The work of \citet{he2024learning} considers the case where offline and online data have different feedback structures for MAB, also proposing algorithms based on UCB. \citet{hao2023leveraging} proposed a Thompson Sampling (TS) algorithm for linear bandits and derived Bayesian regret bounds when the offline data is generated by an expert. The recent work of \citet{vijayan2025regret} proposes an elimination-based algorithm for linear bandits using an extended D-optimal design for exploration, providing near-optimal regret bounds for warm-starting with offline data. In contrast to our work, they solely focus on online regret and do not consider the notion of regret relative to an offline reference. We refer readers to \citet{vijayan2025regret} for an extensive discussion on warm starting online learning using offline data.

\textbf{Online learning with hints.}
Our problem can also be viewed through the lens of online learning with hints, prior knowledge, or side information. \citet{cutkosky2022leveraging} considered the linear bandit problem in which a hint provides a (possibly imperfect) estimate of the optimal action. They proposed an algorithm that improves upon the standard regret guarantees when the hint is accurate while remaining robust to inaccuracies. More recently, \citet{bhaskara2023bandit} investigated a similar problem in the adversarial linear bandit optimization setting. In a different direction, \citet{wei2020taking} studied the contextual bandit problem in both stochastic and adversarial settings where hints are given as predictions of losses, showing that the algorithm can achieve improved regret bounds when the prediction error is small. In all of these works, hints are typically provided as \emph{black-box} side information. This contrasts with our setting, where prior knowledge is available in the structured form of an offline dataset.

\textbf{Conservative exploration.}
In conservative exploration \citep{wu2016conservative,kazerouni2017conservative}, algorithms balance exploration and safety by assuming access to a known safe baseline action. Our offline-to-online learning setting can be cast within this framework by treating the pessimistic action computed from the offline dataset as the safe baseline. This perspective was adopted by \citet{sentenac2025balancing} in the multi-armed bandit (MAB) setting, whose work is most closely related to ours. They propose an algorithm that balances the UCB1 \citep{auer2002finite} and LCB algorithms, and provide a detailed analysis of the tradeoffs between UCB1 and LCB. However, their analysis is restricted to the unstructured MAB setting, in which the regret scales with the number of actions. This dependence can become prohibitive when the action space is large. In contrast, in linear bandits the regret scales with the dimension of the action vectors, which allows for substantially improved scalability in high-dimensional action spaces.

\section{Problem Setting} \label{sec:problem}
\looseness=-1
\paragraph{Notation.} The set $\{1, \dots, n\}$ is denoted by $[n]$ for any $n \in \N$. The set of probability distributions over any set $S$ is denoted by $\cM_1(S)$ and the cardinality of $S$ is denoted by $|S|$. We write the inner product between two vectors $x, y \in \R^d$ as $\ip{x, y} = x^\top y$. For a symmetric positive definite matrix $A$, we denote $\norm{x}_A = \sqrt{x^\top A x} = \sqrt{\ip{Ax, x}}$ for any vector $x$. For two positive semidefinite matrices $A$ and $B$, we write $A \preceq B$ if $B - A$ is positive semidefinite. The trace of a matrix $A$ (sum of its diagonal entries) is denoted by $\tr(A)$ and the determinant of $A$ is denoted by $\det(A)$. Throughout the paper, $\log$ is used to denote the natural logarithm. We use $\tilde{O}(\cdot)$ to hide logarithmic factors in the problem parameters.

\subsection{Linear Bandits with Offline Data}
Let $\cA = \{a_1, \dots, a_k \} \subset \R^d$ be the set of actions (vectors) that the learner can choose and $\cD \coloneqq (\cD_a)_{a \in \cA} $ be the offline dataset.  Similar to previous works \citep{sentenac2025balancing,vijayan2025regret,cheung2024leveraging}, we consider the setting in which the offline data is collected by a fixed design, i.e., the actions are chosen by a non-adaptive policy with deterministic action counts. More concretely, for each $a \in \cA$, $\cD_a \coloneqq ( X_1^{(a)}, \dots, X_{m_a}^{(a)} )$ is the dataset corresponding to action $a$ and $ m_a $ denotes the number of rewards sampled from action $a$. The reward samples $X_i^{(a)} \sim P_a$ for all $i \in [m_a]$, where $\nu \coloneqq (P_a)_{a \in \cA}$ are the reward distributions for each action constituting the \emph{unknown} environment. Given the \emph{offline} dataset $\cD$, the learner interacts \emph{online} with the environment $\nu$ over a sequence of $n \in \N$ rounds. At every round $t \in [n]$, the learner chooses action $A_t \in \cA$, and observes a reward $X_t \sim P_{A_t}$. The goal of the learner is to choose actions to maximize the total reward collected over these $n$ rounds while also leveraging the offline dataset $\cD$ to improve the performance. We make the following standard assumptions on the reward distributions and the actions.

\begin{assumption}[Linear reward function]
    There exists an unknown parameter $\theta_* \in \R^d$ such that $X_t = \ip{\theta_*, A_t} + \eta_t$ and $X_i^{(a)} = \ip{\theta_*, a} + \eta_i^{(a)}$ for all $t \in [n]$, $i \in [m_a]$ and $a \in \cA$, where $\eta_t$ and $\eta_i^{(a)}$ are the noise terms.
\end{assumption}

\begin{assumption}[Subgaussian noise]
    The noise terms $(\eta_t)_{t \in [n]}$ are conditionally $1$-subgaussian given the past online and the offline data. That is, if $\cF_t = \sigma(\cD, A_1, X_1, \dots, A_{t-1}, X_{t-1}, A_t)$ is the $\sigma$-algebra generated by the past online data up to round $t$ and the offline dataset $\cD$, then for all $\lambda \in \R$, we have
    $\EE{\exp(\lambda \eta_t) \mid \cF_t} \le \exp\left( \lambda^2/2 \right)$ almost surely.
    Further, $(\eta_i^{(a)})_{i \in [m_a], a \in \cA}$ are independent $1$-subgaussian, i.e., $\EE{\exp(\lambda \eta_i^{(a)})} \le \exp(\lambda^2/2)$ for all $\lambda \in \R$.
\end{assumption}

\begin{assumption}[Bounded actions, parameters, and rewards]
    There exists constants $B, L > 0$ such that $\norm{\theta_*}_2 \le B$, $\norm{a}_2 \le L$ and $\ip{\theta_*,a} \in [0,1]$ for all $a \in \cA$.
\end{assumption}

\subsection{Performance Metrics}
\looseness=-1
We denote the mean reward of action $a \in \cA$ by $\mu_a \coloneqq \ip{\theta_*, a}$. We measure the performance of any policy $\pi$, which may depend on the history and the offline data, relative to two baselines: (1) the highest mean reward of any arm $\mu_* \coloneqq \max_{a \in \cA} \mu_a$ and (2) the mean reward of the offline dataset $\muoff \coloneqq \sum_{a \in \cA} m_a \mu_a / m$, where $m = \sum_{a \in \cA} m_a$ is the total number of offline reward samples. The pseudo-regret relative to $\mu_*$ is
\begin{align*}
    R_n(\pi, \nu) \coloneqq \sum_{t=1}^n \left( \mu_* - \ip{\theta_*, A_t} \right) \,,
\end{align*}
and the pseudo-regret relative to $\muoff$ is
\begin{align*}
    R_n^{\text{off}}(\pi, \nu) \coloneqq \sum_{t=1}^n \left( \muoff - \ip{\theta_*, A_t} \right) \,.
\end{align*}
Note that both $R_n(\pi, \nu)$ and $R_n^{\text{off}}(\pi, \nu)$ are random quantities because  the sequence of actions $(A_t)_{t \in [n]}$ may be randomized. The pseudo-regret relative to $\mu_*$ is the standard notion of regret while the pseudo-regret $ R_n^{\text{off}}(\pi, \nu) $ captures the loss of performance of $\pi$ compared to our offline reference $\muoff$, induced by the offline sample counts $ (m_a)_{a\in\mathcal A}$. Together, the two metrics characterize different aspects of the offline-to-online tradeoff: $R_n(\pi,\nu)$ measures how far the algorithm is from optimal performance, while $ R_n^{\text{off}}(\pi, \nu) $ quantifies the improvement (or degradation) relative to the historical action distribution represented in the offline dataset. When the dependence on $\nu$ is clear from the context, we write $R_n(\pi)$ and $R_n^{\text{off}}(\pi)$ for brevity.

\subsection{Parameter Estimation}
\looseness=-1
As is common in the linear bandit literature~\citep{abbasi2011improved}, we estimate the true parameter $\theta_*$ using the least squares estimator $\hat \theta_t$ constructed at the \emph{end} of round $t$ using both the offline and the online data:
\begin{align}
    \hat \theta_t = V_t^{-1} \left( \sum_{a \in \cA} \sum_{i=1}^{m_a} X_i^{(a)} a + \sum_{s=1}^t X_s A_s \right) \label{eq:hat_theta_t}\,.
\end{align}
Here, $V_t = V_0 + \sum_{s=1}^t A_s A_s^\top$ is the covariance matrix of both the offline and the online data and $V_0 = \sum_{a \in \cA} m_a aa^\top$ is the covariance matrix corresponding to only the offline data.\footnote{For analysis purposes, we assume that $V_0$ is invertible; if not, we can add a small regularization term to ensure invertibility.}

\section{Method}\label{sec:method}

\looseness=-1
In the offline-to-online setting, the goal is to minimize both the regret against the optimal action, $R_n$, and the regret against the offline reference, $R_n^{\text{off}}$. Upper Confidence Bound (UCB) and Lower Confidence Bound (LCB) methods have been shown to be effective in achieving strong online and offline performance, respectively \citep{abbasi2011improved, li2022pessimism, auer2002finite, xiao2021optimality}. This section presents our algorithm \alg, which carefully switches between UCB and LCB strategies to achieve competitive performance with respect to both $R_n$ and $R_n^{\text{off}}$. 

\looseness=-1
To build confidence bounds, we construct confidence sets $\cC_t$ around the least squares parameter estimate $\hat \theta_t$ given by the following proposition. The proof is deferred to \Cref{app:conf-sets}.

\begin{proposition}\label{prop:total_confidence_set}
    Let $\delta \in (0,1)$. Then with probability at least $1-\delta$, it holds that for all $t \in [n]$
    \begin{align*}
         \norm{\hat \theta_{t-1} - \theta_*}_{V_{t-1}} \le \sqrt{\betaoff} + \sqrt{\beta_t} \,,
    \end{align*}
    where $\sqrt{\betaoff} = \sqrt{8d\log6 + 8\log(2/\delta)}$ is the confidence radius corresponding to the offline data and $\sqrt{\beta_t} = \sqrt{2\log(2/\delta) + \log\left( \det V_{t-1}/\det V_0 \right)}$ with $\beta_1 = 0$ is the confidence radius from both the offline and the online data. Furthermore, for all $t \in [n]$, let
    \begin{align} \label{eq:confidence_set}
        \cC_t = \left\{ \theta \in \R^d \colon \norm{\theta - \hat \theta_{t-1}}_{V_{t-1}} \le \sqrt{\rho_t} \right\} \,,
    \end{align}
    where $\sqrt{\rho_t} \coloneqq \sqrt{\betaoff} + \sqrt{\beta_t}$. Then $\PP(\text{exists } t \in [n] \colon \theta_* \notin \cC_t) \le \delta$.
\end{proposition}
Using the confidence set $\cC_t$, \alg switches between an optimistic action $U_t$ and a pessimistic action $L_t$ given by
\begin{align}\label{eq:ucb-lcb}
    &U_t = \argmax_{a \in \cA} \ucb_t(a) \,, \quad L_t = \argmax_{a \in \cA} \lcb_t(a), \\
    & \text{where } \ucb_t(a) = \sup_{\theta \in \cC_t} \ip{\theta, a} \,, \quad \lcb_t(a) = \max_{t' \le t} \inf_{\theta \in \cC_{t'}} \ip{\theta, a} \nonumber \,.
\end{align}

\looseness=-1
Essentially, the UCB action $U_t$ can reduce the regret $R_n$ by optimism but may increase the regret $R_n^{\text{off}}$, while the LCB action $L_t$ can reduce the regret $R_n^{\text{off}}$ by pessimism but may increase the regret $R_n$. We balance between these two actions to achieve a good performance for both $R_n$ and $R_n^{\text{off}}$.

\subsection{Algorithm} \label{sec:alg}

\looseness=-1
We control the balance between $R_n$ and $R_n^{\text{off}}$ using an \emph{exploration budget} accumulated by choosing the pessimistic action $L_t$. Recall that $R_n^{\text{off}}$ represents the reward deficit of an algorithm relative to $\muoff$, and the LCB strategy can avoid large budget deficits in multi-armed bandits \citep{sentenac2025balancing}. Based on this intuition, we record the algorithm’s excess reward over $\muoff$ and use it as an exploration budget to decide when to choose the optimistic action $U_t$ for exploration. If choosing $U_t$ would not significantly deplete the budget, the algorithm proceeds with exploration; otherwise, it selects $L_t$ to accumulate additional budget before exploring further.

\looseness=-1
This switching mechanism is inspired by conservative bandits \citep{wu2016conservative, kazerouni2017conservative}, and a similar technique was used by \citet{sentenac2025balancing} in the multi-armed offline-to-online bandit setting. 
This budget-based exploration approach crucially relies on constructing a \emph{safe} baseline action—namely, the LCB action in our setting—that is guaranteed to achieve a reward not too far from $\muoff$.
However, the budget gained from the safe LCB action was previously unknown in the linear bandit setting. The following proposition provides the key to constructing this baseline and relies on a technical result that bounds the average uncertainty of actions over the offline dataset. The proof can be found in \Cref{app:proof-of-offline-budget}.

\begin{proposition}\label{prop:loff-subopt}
    Fix $\delta_1 > 0$. Let $\muoff = (1/m)\sum_{a \in \cA} m_a \ip{\theta_*,a}$ and $\lcbaoff = \argmax_{a \in \cA} \lcb_1(a)$ be the LCB action computed from the offline dataset. Then with probability at least $1-\delta_1$,
    \begin{align*}
        \muoff - \mu_{\lcbaoff} \le 2 \sqrt{\frac{d\betaoff}{m}} \,,
    \end{align*}
    where $\betaoff = 8d\log 6 + 8\log(1/\delta_1)$.
\end{proposition}
\looseness=-1
Proposition \ref{prop:loff-subopt} implies that the mean reward of $\lcbaoff$ is not too far from the mean reward of the offline dataset. Since $\lcbaoff$ can be computed directly from the offline data, we construct our baseline reward using a lower bound on $\mu_{\lcbaoff}$ (with some slack). That is, we define our baseline safe reward as
\begin{align} \label{eq:baseline_reward}
    r_b \coloneqq \lcb_1(\lcbaoff) - \alpha S \,,
\end{align}
where $\alpha$ is a tuning parameter and $S\coloneqq\sqrt{d\betaoff/m}$ represents the slack allowed from the lower bound on the offline LCB action's reward. Now we are ready to define the exploration budget $B(t)$ at the beginning of round $t$:
\begin{align}\label{eq:budget}
    B(t) \coloneqq \sum_{s=1}^{t-1} \lcb_t(A_s) + \lcb_t(U_t) - t r_b \,.
\end{align}

\looseness=-1
Our algorithm \alg computes $B(t)$ at the beginning of each round $t$ using the LCB values derived from the confidence set $\cC_t$ (\Cref{prop:total_confidence_set}) and the baseline reward $r_b$ (\cref{eq:baseline_reward}). If $B(t)>0$, choosing $U_t$ will not cause the total reward to drop below $r_b$; therefore, \alg selects $U_t$. Otherwise, it chooses $L_t$ and accumulates an exploration budget of $\lcb_t(A_t)-r_b\ge\lcb_1(\lcbaoff)-r_b=\alpha S$. After executing the chosen action $A_t\in\{U_t,L_t\}$, the algorithm observes a reward sample $X_t\sim P_{A_t}$. Finally, it updates the estimate $\hat{\theta}_t$ using \cref{eq:hat_theta_t} and the covariance matrix $V_t$, and constructs the confidence set $\cC_{t+1}$ for the next round using \cref{eq:confidence_set}. The pseudocode for \alg is presented in \Cref{alg:lin_oto}.

\begin{algorithm}[t]
\caption{Offline-to-Online Algorithm for Linear Bandits (LinOtO)}\label{alg:lin_oto}
\SetKwInput{Input}{Input}

\Input{Offline dataset $\cD$, horizon $n$, baseline tuning parameter $\alpha$}

\BlankLine
Compute $\hat \theta_0$ and $V_0$ using $\cD$ and construct the confidence set $\cC_1$ using \cref{eq:confidence_set}

Compute baseline reward $r_b$ using \cref{eq:baseline_reward}

\For{$t = 1$ \KwTo $n$}{
   Compute the UCB action $U_t$ and the LCB action $L_t$ using $\cC_t$ as in \cref{eq:ucb-lcb} \label{algline:ucb-lcb}

    Compute budget $B(t)$ using \cref{eq:budget}
    
    \If{$B(t) > 0$}{
        
        Choose action $A_t = U_t$
    }
    \Else{Choose action $A_t = L_t$}

    Observe reward $X_t$
    
    Update $V_t$ and $\hat \theta_t$ using \cref{eq:hat_theta_t}, and construct confidence set $\cC_{t+1}$ using \cref{eq:confidence_set}
}
\end{algorithm}

\begin{remark}[Computational complexity]
We make the following remarks on the computational complexity of \alg.
\begin{enumerate}
    \item Computing $U_t$ and $L_t$ on \cref{algline:ucb-lcb} in \Cref{alg:lin_oto} requires solving a bilinear optimization program, which is known to be intractable in general. However, for the finite action set considered in our setting, it can be solved easily. Computation can further be improved by using the Sherman-Morrison formula to incrementally update $V_t^{-1}$. Therefore the per-round computational complexity of computing the actions becomes $O(|\cA| + |\cA|d^2)$ \citep[Chapter 19]{lattimore2020bandit}.

    \item Recall that the LCB value is $\lcb_t(a)=\max_{t'\le t}\inf_{\theta\in\cC_{t'}}\ip{\theta,a}$. We can reduce the computational overhead of evaluating $\lcb_t(a)$ by maintaining a running maximum at each round $t$.

    \item Instead of computing the sum of individual LCBs in \cref{eq:budget}, our implementation computes the LCB of the aggregate action vector $\sum_{s=1}^{t-1} A_s$. This heuristic substantially reduces the computational cost, since the algorithm can incrementally maintain the cumulative action vector $\sum_{s=1}^{t-1} A_s$ without re-evaluating $\lcb_t(A_s)$ for every past round $s$. Empirically, we observe that both implementations exhibit very similar performance.
\end{enumerate}
\end{remark}

\begin{remark}[Unknown horizon]
\Cref{alg:lin_oto} can be adapted to the unknown horizon setting using the doubling trick: run the algorithm in epochs of length $n=2^k$ for $k=0,1,2,\dots$, resetting all historical computations (except the offline estimates) at the end of each epoch. This adapted algorithm enjoys almost identical regret guarantees to \Cref{alg:lin_oto}. For further details, see \citet{sentenac2025balancing}.
\end{remark}

\section{Regret Analysis of LinOtO}\label{sec:regret-analysis}
In this section, we present regret bounds for \alg with respect to both $R_n$ and $R_n^{\text{off}}$, establishing its competitiveness against both LinUCB and LinLCB (the algorithm that chooses the LCB action in every round). The full proofs of all the results in this section are long and therefore deferred to the supplementary section.

\subsection{Regret relative to $\mu_*$}
\looseness=-1
We bound the regret of \alg against the optimal action by decomposing $R_n$ into rounds where the UCB ($U_t$) and LCB ($L_t$) actions are chosen: 
\begin{align*}
    R_n(\alg) &= 
    \sum_{t \in \cU_n} \left( \ip{\theta_*, a_*} - \ip{\theta_*, U_t} \right) + \sum_{t \in \cL_n} \left( \ip{\theta_*, a_*} - \ip{\theta_*, L_t} \right)\\
    &\leq \sum_{t \in \cU_n} \left( \ip{\theta_*, a_*} - \ip{\theta_*, U_t} \right) + \left|\cL_n\right|\max_{t \in \cL_n}\left( \ip{\theta_*, a_*} - \ip{\theta_*, L_t} \right)\\
    &\leq \sum_{t \in \cU_n} \left( \ip{\theta_*, a_*} - \ip{\theta_*, U_t} \right) + 2\sqrt{\betaoff}\left|\cL_n\right| \norm{a_*}_{V_0^{-1}}\;,
\end{align*}
where $\cU_n$ and $\cL_n$ denote the rounds where the learner chooses UCB and LCB actions, respectively.
In the last line, we used the following suboptimality bound for LCB, which is a special case of \citet[Theorem 1]{li2022pessimism}.
\begin{lemma}[LCB suboptimality w.r.t optimal action]\label{lem:lcb-a*}
Suppose that the event $\cE = \{\theta_* \in \cC_t \text{ for all } t \in [n]\}$ holds. Then for any $t \in [n]$, 
    \begin{align*}
        \ip{\theta_*, a_*} - \ip{\theta_*, L_t} \le 2 \sqrt{\betaoff}\norm{a_*}_{V_0^{-1}} \,.
    \end{align*}
\end{lemma}
\looseness=-1
The regret during $\cU_n$ is bounded via a LinUCB analysis adapted for offline data. Thus, the key is to control the regret during $\cL_n$ by bounding the total number of times the LCB action is chosen, namely, $|\cL_n|$. 

\paragraph{Bounding $|\cL_n|$.}
\looseness=-1
By \cref{eq:budget}, \alg chooses $L_t$ only to cover budget deficits ($B(t) < 0$), generating an $\alpha S$ surplus that is then consumed by $U_t$. As the confidence bounds shrink over time (from \cref{eq:confidence_set}), the lower confidence bound of $U_t$ will eventually exceed the baseline $r_b$. Once this occurs, playing $U_t$ becomes self-funding, and $L_t$ will no longer be chosen. By equating the budget deficit incurred by $U_t$ with the $\alpha S$ increments, we can bound the total number of LCB actions by:
\begin{align*}
    |\cL_n| = \tilde O \left( \frac{d^2}{\alpha S ( \Delta_{\text{off}} + \alpha S)} \right) \,,
\end{align*}
where $\Delta_{\text{off}} = \ip{\theta_*, a_* - \lcbaoff}$. For the full proof, see \Cref{app:proof-linoto-reg}. 

\looseness=-1
Finally, to state the regret bound, we define an ``effective dimension'' \citep{valko2014spectral}:
\begin{align*}
    d_{\text{eff}} = \max \set{ i \in [d] \,:\, (i-1) \lambda_i(V_0) \le \frac{nL^2}{\log(1+nL^2/\lambda_1(V_0))}} \,,
\end{align*}
where $\lambda_1(V_0) \le \dots \le \lambda_d(V_0)$ are the eigenvalues of the offline covariance matrix $V_0$. The effective dimension $d_{\text{eff}}$ represents the maximum number of directions $i \in [d]$ for which the corresponding eigenvalues $\lambda_j(V_0)$ (for $j \le i$) are small. It captures the maximum number of directions where the offline data lacks sufficient coverage. Note that $d_{\text{eff}} \le d$.

\looseness=-1
By combining the above discussions, we have the following regret bound on $R_n$.
\begin{theorem}\label{thm:linoto-reg}
    With probability at least $1 - \delta$, the pseudo-regret of \alg is bounded as
    \begin{align*}
    R_n(\alg) = \tilde{O}\left( \sqrt{n \rho_n d_{\text{eff}} \log\left(1 + \frac{nL^2}{\lambda_{\min}(V_0)} \right)} + \frac{d^{2.5} \norm{a_*}_{V_0^{-1}}}{\alpha S (\Delta_{\text{off}} + \alpha S) } \right) \,.
\end{align*}
When $\lambda_{\min}(V_0) = \Omega(mL^2/d)$ and $\Delta_{\text{off}} + \alpha S \gtrsim \sqrt{\rho_n d/m}$, then with probability $1-\delta$, $R_n(\alg) = O\left( n \sqrt{\rho_n dd_{\text{eff}}/{m} }\right)$.
\end{theorem}

\begin{remark}[Comparison with LinUCB]
 The regret bound of \alg consists of two terms: the first corresponds to the regret due to UCB actions, and the second corresponds to the regret from LCB actions. The first term is sublinear in the horizon $n$ and depends on the offline data through $d_{\text{eff}}$. If the offline data is not well-covered ($d_{\text{eff}} = d$), we recover the standard $\tilde{O}(d\sqrt{n})$ bound for LinUCB~\citep{abbasi2011improved}. However, when there is good offline coverage with large $m$, the standard LinUCB bound improves to $O\left( n\sqrt{\rho_n dd_{\text{eff}}/m} \right)$. Moreover, the second term is independent of $n$ (up to log factors), and hence not growing with online data. Therefore, overall, the regret of \alg is comparable to LinUCB up to an additive constant.
\end{remark}

\begin{remark}[Analysis of $\alpha$]\label{rem:alpha}

When $\alpha$ is large, the budget is more likely to be positive and \alg behaves closer to LinUCB. From \Cref{thm:linoto-reg}, the regret bound in this case is similar to that of LinUCB. When $\alpha$ is small, then \alg mostly plays the LCB actions and its regret corresponding to the LCB plays increases as shown in \Cref{thm:linoto-reg}. A more refined analysis of how to tune $\alpha$ is left for future work.

\end{remark}

\subsection{Regret relative to $\muoff$}

Now we turn to $R_n^{\text{off}}$, the regret of \alg relative to the offline reference $\muoff$. The intuition here is that choosing the LCB action incurs a constant regret per round (from \cref{prop:loff-subopt} and the monotonicity of LCB values), and the confidence widths decrease whenever the UCB action is chosen (from \cref{eq:confidence_set}). Therefore, the cumulative regret relative to $\muoff$ is of the same order as the offline optimal LinLCB.
\begin{theorem}\label{thm:linoto-offreg}
    With probability at least $1 - \delta$, we have
    \begin{align*}
        R_n^{\text{off}}(\alg) = O\left( (\alpha+2) \sqrt{\frac{d \rho_1}{m}} n \right)\,.
    \end{align*}
\end{theorem}
To compare the performance of \alg against LinLCB, we next state the regret of LinLCB relative to $\muoff$.
\begin{proposition}\label{prop:lcb-offreg}
    With probability $1 - \delta$, we have
    \begin{align*}
        R_n^{\text{off}}(\text{LinLCB}) =  O\left( \sqrt{\frac{d \rho_n}{m}}n \right) \,.
    \end{align*}
\end{proposition}
From \Cref{thm:linoto-offreg} and \Cref{prop:lcb-offreg}, we observe that the regret of \alg is of the same order as LinLCB as $\rho_1$ and $\rho_n$ differ only by a factor of $\log n$. Note that $R_n^{\text{off}}$ scales linearly with the online horizon $n$ because LinLCB does not explore. However, the per-round regret vanishes as $O(1/\sqrt{m})$. Moreover, there is a Pareto tradeoff between $R_n(\alg)$ and $R_n^{\text{off}}(\alg)$ in terms of $\alpha$. For large $\alpha$, say $O(\sqrt{m})$, the second term in \Cref{thm:linoto-reg} becomes smaller improving the online regret $R_n$. However, $R_n^{\text{off}}$ becomes a constant in $m$. On the other hand, setting $\alpha = O(1)$ preserves the $1/\sqrt{m}$ decay of $R_n^{\text{off}}$ at the cost of a larger $R_n$ bound.

\section{Experimental Results}\label{sec:experiments}
\looseness=-1
In this section, we evaluate the empirical performance of \alg on synthetic linear bandit tasks to demonstrate its competitiveness against  warm-started LinUCB, LinLCB, and a random policy. For the default setting, we generate $|\cA| = 100$ actions uniformly over unit sphere in dimension $d=20$. The true parameter is also generated randomly over the unit sphere and $\alpha$ is set to $10^{-4}$. Both the offline and the online reward noises are sampled from standard normal distribution. The offline dataset is constructed with $m=100$ samples evenly distributed over 10 suboptimal actions. We add a small regularization term of $10^{-4}$ to the offline covariance matrix to ensure invertibility.

\begin{figure}[t!]
    \centering
    \begin{subfigure}{\textwidth}
        \centering
        \includegraphics[width=\linewidth]{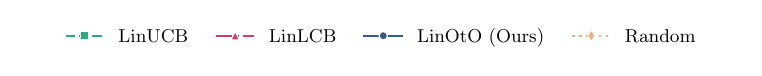}
        \label{fig:legend}
    \end{subfigure}
    
    \begin{subfigure}{\textwidth}
        \centering
        \includegraphics[width=\linewidth]{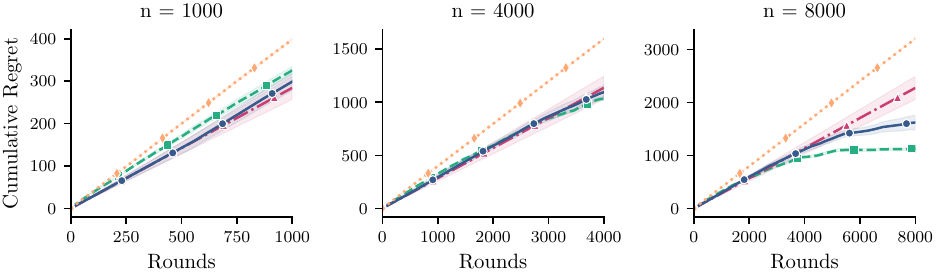}
        \caption{Cumulative regret across varying horizon lengths. For a smaller horizon ($n=1000$), the performance of \alg closely tracks LinLCB due to limited initial exploration. However, for a larger horizon ($n=8000$), \alg successfully explores and achieves a sublinear regret profile matching LinUCB.}
        \label{fig:n}
    \end{subfigure}

    \vspace{0.8em} 

    \begin{subfigure}{\textwidth}
        \centering
        \includegraphics[width=\linewidth]{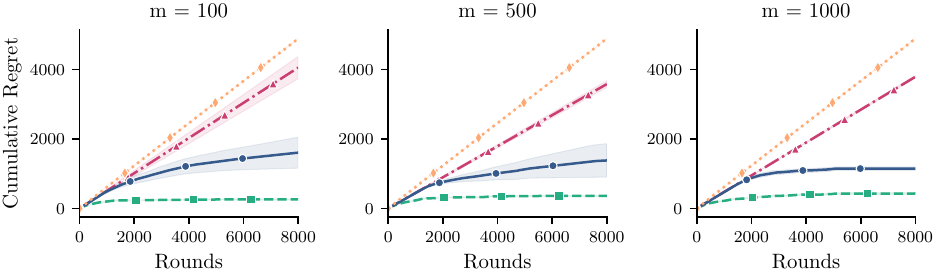}
        \caption{Cumulative regret for varying offline dataset sizes with no optimal actions in the offline dataset. We vary the number of offline samples $m$ (excluding samples of the optimal action) while fixing the horizon to $n=8000$. For $m=1000$, we set $\alpha=0.01$. \alg achieves a regret bounded between LinUCB and LinLCB, converging to the performance of LinUCB for large $m$}
        \label{fig:m_subopt}
    \end{subfigure}

    \vspace{0.8em}

    \begin{subfigure}{\textwidth}
        \centering
        \includegraphics[width=\linewidth]{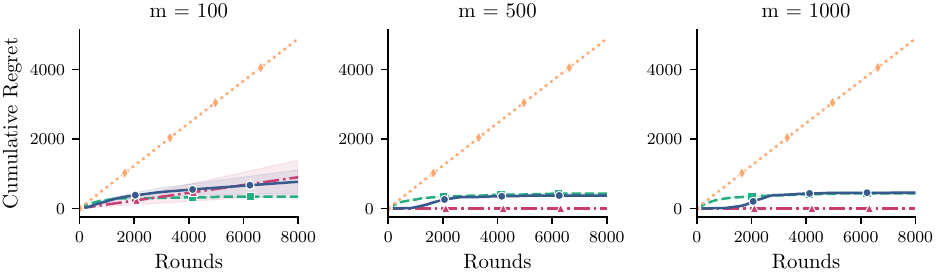}
        \caption{Cumulative regret for varying offline sample size $m$ with optimal actions in the offline dataset. LinLCB performs best by immediately exploiting the high-quality offline data. Both LinUCB and \alg incur some exploration penalty but quickly converge.}
        \label{fig:m_opt}
    \end{subfigure}
    \caption{Cumulative regret of LinUCB, LinLCB, \alg, and a random policy (Random) as a function of number of online rounds. The shaded area is the standard error measured over 5 independent runs for all the algorithms.
    }
    \label{fig:sim_results}
\end{figure}

\looseness=-1
\paragraph{Varying the Horizon Length.} In \Cref{fig:n}, we fix the offline dataset size to $m=100$ and vary the online horizon length $n \in \{1000, 4000, 8000\}$ which is an input to \Cref{alg:lin_oto}, and plot the cumulative pseudo-regret $R_n$ over rounds. We observe that in the early rounds, \alg successfully leverages the offline data to maintain a low initial regret, closely tracking the safe baseline of LinLCB. As $n$ increases and the exploration budget permits, \alg transitions to online exploration. It eventually achieves a sublinear regret similar to LinUCB. In contrast, LinLCB continues to follow the conservative baseline and therefore incurs linear regret.

\looseness=-1
\paragraph{Varying Offline Data Size (Excluding Optimal Action).}
Next, we examine the impact of the offline data size $m \in \{100, 500, 1000\}$ in a setting where the dataset excludes the optimal action. In this regime, the offline reference action $\lcbaoff$ can be strictly suboptimal. \Cref{fig:m_subopt} plots the cumulative regret for $n = 8000$ with different offline sample sizes $m$. Because LinLCB does not explore, it commits to a suboptimal safe action and therefore incurs linear regret. In contrast, \alg uses its exploration budget to safely deviate from the suboptimal offline prior. Although its performance initially resembles that of LinLCB, it eventually explores and achieves sublinear regret. When $m=1000$, the performance of \alg improves further and becomes closer to that of LinUCB, consistent with the theoretical insights discussed in \Cref{sec:regret-analysis}.

\looseness=-1
\paragraph{Varying Offline Data Size (Including Optimal Action).}
In \Cref{fig:m_opt}, we evaluate the performance of \alg when the offline dataset contains samples from the optimal action. The performance of all algorithms (except Random) improves as the number of offline samples $m$ increases. As expected, LinLCB quickly identifies the optimal action and achieves (near) zero regret. Both \alg and LinUCB perform a small amount of exploration but rapidly converge to the optimal action. Additional experimental results are provided in \Cref{app:more-expts}.

\section{Conclusion}
In this work, we studied the problem of offline-to-online learning in the stochastic linear bandit model. By carefully switching between the LinUCB and LinLCB strategies, we developed the algorithm \alg, which is designed to achieve strong performance relative to both online learning and the offline baseline. Theoretically, we established regret bounds showing that \alg is simultaneously competitive with both LinUCB and LinLCB. We also validated these theoretical guarantees through empirical evaluations across a range of environmental parameters.

A primary limitation of the current framework is the assumption of fixed-design offline data, rather than data collected through an adaptive logging policy. Moreover, the offline reference does not account for the coverage properties of the offline dataset. Consequently, the offline regret bounds scale with the total number of offline samples rather than with a more refined coverage-dependent quantity. Addressing these limitations remains an important open problem. Additional future directions include studying distribution shift between offline data and online interactions, as well as extending the framework to nonlinear function approximation settings.

\appendix

\section{Confidence Sets Construction}\label{app:conf-sets}
We start by constructing the confidence set $\cC_1$ from the offline data. The least squares estimate $\hat \theta_0$ of $\theta_*$ is given by
\begin{align*}
    \hat \theta_0 = V_0^{-1} \sum_{a \in \cA} \sum_{i=1}^{m_a} X_i^{(a)} a \,,
\end{align*}
where $V_0 = \sum_{a \in \cA} m_a a a^\top$ is the offline covariance matrix. To get the confidence set around $\hat{\theta}_0$, we use the following lemma, which can be obtained using concentration of subgaussian random variables and a covering argument \citep[Equation 20.3]{lattimore2020bandit}.
\begin{lemma}[\citet{lattimore2020bandit}]\label{lem:off-conc}
    For all $\delta_0 \in (0,1)$, we have
    \begin{align*}
        \PP\left( \norm{\sum_{a \in \cA} \sum_{i=1}^{m_a} \eta_i^{(a)} a}_{V_0^{-1}} \ge \sqrt{\betaoff(\delta_0)} \right) \le \delta_0, 
    \end{align*}
\end{lemma}
where $\sqrt{\betaoff(\delta_0)} \coloneqq \sqrt{8d\log6 + 8\log(1/\delta_0)}$ is the offline confidence radius. The next lemma gives a uniform concentration guarantee for all $t \in [n]$ based on a martingale argument \citep[Theorem 20.4]{lattimore2020bandit}.
\begin{lemma}[\citet{lattimore2020bandit}]\label{lem:mart-conc}
For all $\delta_1 \in (0,1)$,
\begin{align*}
    \PP\left( \text{exists } t > 1 \,:\, \norm{\sum_{s=1}^t \eta_s A_s}_{V_t^{-1}} \ge \sqrt{\beta_{t+1}(\delta_1)} \right) \le \delta_1, 
\end{align*}
where $\sqrt{\beta_t(\delta_1)} \coloneqq \sqrt{2\log\left(2/\delta_1 \right) + \log\left( \det V_{t-1}/\det V_0 \right)}$ with $V_t = V_0 + \sum_{s=1}^t A_s A_s^\top$.
\end{lemma}
Combining \cref{lem:off-conc,lem:mart-conc}, we get
\begin{lemma}\label{lem:on-off-conc}
    For all $\delta \in (0, 1)$, there exists $\delta_0, \delta_1 \in (0,1)$ such that
    \begin{align*}
        \PP\left( \text{exists } t > 1 \,:\, \norm{\sum_{s=1}^t \eta_s A_s}_{V_t^{-1}} \ge \beta_{t+1}(\delta_1) \, \text{ or } \, \norm{\sum_{a \in \cA} \sum_{i=1}^{m_a} \eta_i^{(a)} a}_{V_0^{-1}} \ge \betaoff(\delta_0) \right) \le \delta \,.
    \end{align*}
\end{lemma}
\begin{proof}
The proof follows from \cref{lem:off-conc,lem:mart-conc} by choosing $\delta_0 = \delta_1 = \delta/2$ and a union bound.
\end{proof}
\subsection{Proof of \cref{prop:total_confidence_set}}
\begin{proof}
    Using the definition of $X_i^{(a)}$ and $X_s$, we have
    \begin{align*}
        \hat \theta_t - \theta_* &= V_t^{-1} \left( \sum_{a \in \cA} \sum_{i=1}^{m_a} X_i^{(a)}a + \sum_{s=1}^t X_s A_s \right) - \theta_* \\
        &= V_t^{-1} \left( \sum_{a \in \cA} \sum_{i=1}^{m_a} \left( \ip{\theta_*, a} + \eta_i^{(a)} \right) a + \sum_{s=1}^t (\ip{\theta_*, A_s} + \eta_s) A_s \right) - \theta_* \\
        &= V_t^{-1} \left( \sum_{a\in \cA} m_a aa^\top \theta_* + \sum_{s=1}^t A_s A_s^\top \theta_* + \sum_{a \in \cA} \sum_{i=1}^{m_a} \eta_i^{(a)} a + \sum_{s=1}^t \eta_s A_s \right) - \theta_* \\
        &= V_t^{-1} \left( V_t \theta_* + \sum_{a \in \cA} \sum_{i=1}^{m_a} \eta_i^{(a)} a + \sum_{s=1}^t \eta_s A_s \right) - \theta_* \\
        &= V_t^{-1} \left( \sum_{a \in \cA} \sum_{i=1}^{m_a} \eta_i^{(a)} a + \sum_{s=1}^t \eta_s A_s \right) \,.
    \end{align*}
    Noting that $\norm{B^{-1} x}_B = \norm{x}_{B^{-1}}$ for any positive definite matrix $B$, we have 
    \begin{align*}
        \norm{\hat \theta_t - \theta_*}_{V_t} &= \norm{\sum_{a\in\cA}\sum_{i=1}^{m_a}\eta_i^{(a)} a +\sum_{s=1}^t\eta_s A_s}_{V_t^{-1}} \\
        &\le \norm{\sum_{a\in\cA}\sum_{i=1}^{m_a}\eta_i^{(a)} a}_{V_t^{-1}} + \norm{\sum_{s=1}^t\eta_s A_s}_{V_t^{-1}} \\
        &\le \norm{\sum_{a\in\cA}\sum_{i=1}^{m_a}\eta_i^{(a)} a}_{V_0^{-1}} + \norm{\sum_{s=1}^t\eta_s A_s}_{V_t^{-1}} \tag{since $V_0 \preceq V_t$}\\
        &\le \sqrt{\rho_t} \,.
    \end{align*}
where the last step holds with probability at least $1-\delta$ from \cref{lem:on-off-conc}.
\end{proof}

\section{Baseline Construction}
Our baseline construction using \Cref{prop:loff-subopt} requires the following technical lemma.
\begin{lemma}\label{lem:norm_bound}
    Let $V_0 = \sum_{a \in \cA} m_a a a^\top$ and $m = \sum_{a \in \cA} m_a$. Then 
    \begin{align*}
        \frac{1}{m} \sum_{a \in \cA} m_a \norm{a}_{V_0^{-1}} \le \sqrt{\frac{d}{m}}.
    \end{align*}
\end{lemma}
\begin{proof}
    First, using Cauchy-Schwarz inequality and $\sum_{a \in \cA} m_a/m = 1$, we have
    \begin{align*}
        \sum_{a \in \cA} \frac{m_a}{m} \norm{a}_{V_0^{-1}} \le \sqrt{\sum_{a \in \cA}  \frac{m_a}{m}} \sqrt{\sum_{a \in \cA} \frac{m_a}{m} \norm{a}_{V_0^{-1}}^2} = \sqrt{\sum_{a \in \cA} \frac{m_a}{m} \norm{a}_{V_0^{-1}}^2} \,.
    \end{align*}
    Using the cyclicity of trace, the last expression inside the square root can be rewritten as
    \begin{align*}
         \tr\left(\sum_{a \in \cA} \frac{m_a}{m} a^\top V_0^{-1} a \right) = \tr \left( V_0^{-1} \frac1m \sum_{a \in \cA} m_a a a^\top \right) = \frac1m \tr\left( V_0^{-1} V_0 \right) = \frac{d}{m}\,.
    \end{align*}
    Therefore $\sum_{a \in \cA} (m_a/m) \norm{a}_{V_0^{-1}} \le \sqrt{d/m}$.
\end{proof}

\subsection{Proof of \cref{prop:loff-subopt}}\label{app:proof-of-offline-budget}
\begin{proof}
    Since $\theta_* \in \cC_1$ with high probability, we have
\begin{align*}
    \mu_{\lcbaoff} &\ge \lcb_1(\lcbaoff) \\
    &\ge \frac1m \sum_{a \in \cA} m_a \inf_{\theta \in \cC_1} \ip{\theta, a} \tag{since max $\ge$ avg}\\
    &= \frac1m \sum_{a \in \cA} m_a \inf_{\theta \in \cC_1} \left( \ip{\theta_*, a} - \ip{\theta_* - \theta, a} \right) \\
    &= \muoff - \frac1m \sum_{a \in \cA} m_a \sup_{\theta \in \cC_1} \ip{\theta_* - \theta, a} \\
    &\ge \muoff - \frac{2\sqrt{\betaoff}}{m} \sum_{a \in \cA} m_a \norm{a}_{V_0^{-1}} \tag{by Cauchy-Schwarz} \\
    &\ge \muoff - 2\sqrt{\frac{d\betaoff}{m}}. \tag{by \cref{lem:norm_bound}}
\end{align*}
Rearranging the above proves the result.
\end{proof}

\bibliography{main}
\bibliographystyle{rlj}

\beginSupplementaryMaterials

\section{Regret Analysis of \alg}\label{app:proof-linoto-reg}
Our analysis is based on the proof of linear conservative bandits \citep{kazerouni2017conservative} with some key modifications since our baseline is chosen based on the offline data. We will assume that the event $\cE = \{\theta_* \in \cC_t \text{ for all } t \in [n]\}$ holds throughout, which happens with probability at least $1-\delta$ by \cref{prop:total_confidence_set}. 

\subsection{Proof of \cref{thm:linoto-reg}}
\begin{proof}
    Let $\cU_t$ be the set of rounds till round $t$ when the budget was non-negative, that is, when \alg chooses the UCB action. Similarly, let $\cL_t$ be the set of rounds when \alg chooses the LCB action. Then the regret of \alg can be decomposed as
\begin{align*}
    R_n(\alg) &= \sum_{t=1}^n \ip{\theta_*, a_*} - \ip{\theta_*, A_t} \\
    &= \sum_{t \in \cU_n} \left( \ip{\theta_*, a_*} - \ip{\theta_*, U_t} \right) + \sum_{t \in \cL_n} \left( \ip{\theta_*, a_*} - \ip{\theta_*, L_t} \right) \\
    &\le \sum_{t \in \cU_n} \left( \ip{\theta_*, a_*} - \ip{\theta_*, U_t} \right) + 2 \sqrt{\betaoff} \norm{a_*}_{V_0^{-1}} | \cL_n |. \tag{by \cref{lem:lcb-a*}}
\end{align*}
Now we bound $| \cL_n |$. Let $\tau \coloneqq \max\{ t \in [n] \,:\, A_t = L_t \}$ be the last round when the LCB action is chosen. Then $\cL_\tau = \cL_n$. Since $L_\tau$ was chosen in round $\tau$, we also have $B(\tau) < 0$. That is,
\begin{align*}
    \sum_{t=1}^{\tau-1} \lcb_\tau(A_t) + \lcb_\tau(U_\tau) - \tau r_b < 0 \,.
\end{align*}
Decomposing the sum over $\tau - 1$ into rounds when UCB and LCB actions are chosen, and distributing $\tau r_b$ over these two sets of rounds, the inequality in the last display can be rewritten as
\begin{align} \label{eq:budget_decomposition}
    \sum_{t \in \cU_{\tau-1}} \left( \lcb_\tau(A_t) - r_b \right) + \sum_{t \in \cL_{\tau-1}} \left( \lcb_\tau(A_t) - r_b \right) + \lcb_\tau(U_\tau) - r_b < 0 \,.
\end{align}
For all $t \in \cL_{\tau-1}$, the second term above can be lower bounded as follows:
\begin{align*}
    \lcb_\tau(A_t) - r_b &= \lcb_\tau(L_t) - r_b \tag{since $A_t = L_t$ for all $t \in \cL_{\tau-1}$}\\
    &\ge \lcb_t(L_t) - r_b \tag{since LCB values are non-decreasing} \\
    &\ge \lcb_t(\lcbaoff) - rb \tag{since $L_t = \argmax_{a \in \cA} \lcb_\tau(a)$} \\
    &\ge \lcb_1(\lcbaoff) - r_b \tag{since LCB values are non-decreasing} \\
    &= \alpha S \tag{by the definition of $r_b$ in \cref{eq:baseline_reward}} \,.
\end{align*}
Therefore, \cref{eq:budget_decomposition} becomes
\begin{align*}
    \sum_{t \in \cU_{\tau-1}} \left( \lcb_\tau(A_t) - r_b \right) + \sum_{t \in \cL_{\tau-1}} \alpha S + \lcb_\tau(U_\tau) - r_b < 0 \,.
\end{align*}
Since $\sum_{t \in \cL_{\tau-1}} \alpha S = |\cL_{\tau-1}| \alpha S$, we have\
\begin{align}\label{eq:lcb_iter_main}
    | \cL_{\tau-1} | \alpha S &< r_b - \lcb_\tau(U_\tau) + \sum_{t \in \cU_{\tau-1}} \left( r_b - \lcb_\tau(A_t) \right) \,.
\end{align}
Let $\tilde \theta_t \coloneqq \argmax_{\theta \in \cC_t} \ip{\theta, U_t}$ be the optimistic estimate of $\theta_*$ in round $t$. Then for all $t \in \cU_{\tau-1}$, we have 
\begin{align}
    \lcb_\tau(A_t) &= \lcb_\tau(U_t) \nonumber \\
    &= \max_{t' \le \tau} \inf_{\theta \in \cC_{t'}} \ip{\theta, U_t} \nonumber \\
    &= \max_{t' \le \tau} \inf_{\theta \in \cC_{t'}} -\ip{\theta_* - \theta, U_t} + \ip{\theta_*, U_t} \nonumber \\
    &= \max_{t' \le \tau} \inf_{\theta \in \cC_{t'}} -\ip{\theta_* - \theta, U_t} - \ip{\tilde \theta_t - \theta_*, U_t} + \ip{\tilde \theta_t, U_t} \label{eq:bounded_r_ellip} \\
    &\ge \max_{t' \le \tau} \inf_{\theta \in \cC_{t'}} -\norm{\theta_* - \theta}_{V_{t-1}} \norm{U_t}_{V_{t-1}^{-1}} - \norm{\tilde \theta_t - \theta_*}_{V_{t-1}} \norm{U_t}_{V_{t-1}^{-1}} + \ip{\tilde \theta_t, U_t} \nonumber \\
    &\ge \inf_{\theta \in \cC_t} - \norm{\theta_* - \theta}_{V_{t-1}} \norm{U_t}_{V_{t-1}^{-1}} - \norm{\tilde \theta_t - \theta_*}_{V_{t-1}} \norm{U_t}_{V_{t-1}^{-1}} + \ip{\tilde \theta_t, U_t} \nonumber \\
    &\ge -2\sqrt{\rho_t} \norm{U_t}_{V_{t-1}^{-1}} - 2\sqrt{\rho_t} \norm{U_t}_{V_{t-1}^{-1}} + \ip{\tilde \theta_t, U_t} \nonumber \\
    &\ge -4 \sqrt{\rho_\tau}\norm{U_t}_{V_{t-1}^{-1}} + \ip{\tilde \theta_t, U_t} \tag{since $\rho_\tau \ge \rho_t$ for all $t \le \tau$} \nonumber \,.
\end{align}
Since the mean rewards are between $[0,1]$, we can also bound $\lcb_\tau(A_t)$ from \Cref{eq:bounded_r_ellip} as $\lcb_\tau(A_t) \ge -4\sqrt{\rho_\tau} + \ip{\tilde \theta_t, U_t}$ to get
\begin{align*}
    \lcb_\tau(A_t) \ge -4\sqrt{\rho_\tau} \left(1 \land \norm{U_t}_{V_{t-1}^{-1}} \right) + \ip{\tilde \theta_t, U_t} \,,
\end{align*}
where $a \land b \coloneqq \min\{a, b\}$. Similarly, we can also bound $\lcb_\tau(U_\tau)$ as
\begin{align*}
    \lcb_\tau(U_\tau) = \max_{t' \le \tau} \inf_{\theta \in \cC_{t'}} \ip{\theta, U_\tau} \ge -4 \sqrt{\rho_\tau} \left(1 \land \norm{U_\tau}_{V_{\tau-1}^{-1}} \right) + \ip{\tilde \theta_\tau, U_\tau} \,.
\end{align*}
Substituting the previous two displays in \cref{eq:lcb_iter_main} and rearranging, we get 
\begin{align}\label{eq:lcb_iter_main_2}
    | \cL_{\tau-1} | \alpha S &< \sum_{t \in \cU_{\tau-1} \cup \{\tau\}} \left( r_b - \ip{\tilde \theta_t, U_t} + 4\sqrt{\rho_\tau} \left(1 \land \norm{U_t}_{V_{t-1}^{-1}} \right) \right) \,.
\end{align}

Note that for any $t \in [n]$,
\begin{align*}
    r_b - \ip{\tilde \theta_t, U_t} &= \lcb_1(\lcbaoff) - \alpha S - \ip{\tilde \theta_t, U_t}  \\
    &\le \ip{\theta_*, \lcbaoff} - \ip{\theta_*, a_*} - \alpha S \\
    &= -\Delta_{\text{off}} - \alpha S \,,
\end{align*}
where $\Delta_{\text{off}} = \ip{\theta_*, a_* - \lcbaoff}$ is the suboptimality gap of the offline LCB action $\lcbaoff$. So \cref{eq:lcb_iter_main_2} can be rewritten as
\begin{align*}
    | \cL_{\tau-1} | \alpha S &< -(\Delta_{\text{off}} + \alpha S)\left( |\cU_{\tau-1}|+1 \right)  + 4 \sqrt{\rho_\tau} \sum_{t \in \cU_{\tau-1} \cup \{\tau\}} \left(1 \land \norm{U_t}_{V_{t-1}^{-1}} \right) \,.
\end{align*}
Using Cauchy-Schwarz and the elliptical potential lemma \citep{abbasi2011improved,kazerouni2017conservative}, we have
\begin{align*}
    \sum_{t \in \cU_{\tau-1} \cup \{\tau\}} \left(1 \land \norm{U_t}_{V_{t-1}^{-1}} \right) &\le \sqrt{2\left(|\cU_{\tau-1}|+1 \right) d \log\left(1 + \frac{\left(| \cU_{\tau-1}|+1 \right)L^2}{d\det(V_0)^{1/d}}\right)} \,,
\end{align*}
which can be further bounded using $\det V_0 \ge (\lambda_{\min}(V_0))^d$ as
\begin{align*}
    \sum_{t \in \cU_{\tau-1} \cup \{\tau\}} \left(1 \land \norm{U_t}_{V_{t-1}^{-1}} \right) &\le \sqrt{2\left(|\cU_{\tau-1}|+1 \right) d \log\left(1 + \frac{\left(| \cU_{\tau-1}|+1 \right)L^2}{d \lambda_{\min}(V_0)}\right)} \,.
\end{align*}
Therefore
\begin{align}
    | \cL_{\tau-1} | \alpha S &< -(\Delta_{\text{off}} + \alpha S)\left( |\cU_{\tau-1}| + 1 \right) \nonumber \\
    & \quad + 4\sqrt{\rho_\tau} \sqrt{2 d\left( |\cU_{\tau-1}|+1 \right) \log\left(1 + \frac{\left(| \cU_{\tau-1}|+1 \right)L^2}{d \lambda_{\min}(V_0)}\right)} \nonumber \\
    & \le -(\Delta_{\text{off}} + \alpha S)\left( |\cU_{\tau-1}| + 1 \right) \nonumber \\
    & \quad + 4\sqrt{2 \rho_n \left(|\cU_{\tau-1}|+1 \right) d \log\left(1 + \frac{\left(| \cU_{\tau-1}|+1 \right)L^2}{d \lambda_{\min}(V_0)}\right)} \label{eq:lcb_rounds}\,.
\end{align}
Since $|\cU_{\tau-1}|+1 \le n$, the RHS of \Cref{eq:lcb_rounds} becomes a quadratic function in $\sqrt{|\cU_{\tau-1}|+1}$ and can be bounded by its maximum value. Therefore,
\begin{align*}
    |\cL_{\tau-1}| = \tilde O \left( \frac{d^2}{\alpha S ( \Delta_{\text{off}} + \alpha S)} \right) \,.
\end{align*}
Finally, since $|\cL_n| = |\cL_\tau| = |\cL_{\tau-1}|+1$, the regret corresponding to the LCB plays for Case 1 can be bounded as
\begin{align*}
    \sum_{t \in \cL_n} \left( \ip{\theta_*, a_*} - \ip{\theta_*, L_t} \right) &\le 2 \sqrt{\betaoff} \norm{a_*}_{V_0^{-1}} |\cL_{\tau-1}| \\
    &= \tilde O \left( \frac{d^{2.5} \norm{a_*}_{V_0^{-1}}}{\alpha S ( \Delta_{\text{off}} + \alpha S)} \right) \,.
\end{align*}
To bound the regret corresponding to the UCB plays, we can adapt the standard LinUCB analysis \citep{abbasi2011improved,lattimore2020bandit} and of \citep{valko2014spectral} to get
\begin{align}
    \sum_{t \in \cU_n} \left( \ip{\theta_*, a_*} - \ip{\theta_*, U_t} \right) &\le \sqrt{8 n \rho_n d_{\text{eff}} \log\left(1 + \frac{nL^2}{\lambda_{\min}(V_0)} \right)}\,, \label{eq:warm_ucb}
\end{align}
Therefore, with probability $1-\delta$, the total regret of \alg can be bounded as
\begin{align*}
    R_n(\alg) = \tilde{O}\left( \sqrt{n \rho_n d_{\text{eff}} \log\left(1 + \frac{nL^2}{\lambda_{\min}(V_0)} \right)} + \frac{d^{2.5} \norm{a_*}_{V_0^{-1}}}{\alpha S ( \Delta_{\text{off}} + \alpha S)} \right) \,.
\end{align*}
When $\lambda_{\min}(V_0) = \Omega(mL^2/d)$, using $\log(1+x) \le x$, we can upper bound \Cref{eq:lcb_rounds} as
\begin{align*}
    | \cL_{\tau-1} | \alpha S & \le -(\Delta_{\text{off}} + \alpha S)\left( |\cU_{\tau-1}| + 1 \right) \nonumber \\
    & \quad + 4\left(|\cU_{\tau-1}|+1 \right)\sqrt{2 d\rho_n/m} \,,
\end{align*}
which can be bounded by 0 for large enough $\alpha$ satisfying $\Delta_{\text{off}} + \alpha S \ge 4\sqrt{2d\rho_n/m}$. In this case, the regret for the UCB plays can be upper bounded using \Cref{eq:warm_ucb} as
\begin{align*}
    \sum_{t \in \cU_n} \left( \ip{\theta_*, a_*} - \ip{\theta_*, U_t} \right) &= O\left( n \sqrt{\frac{\rho_n dd_{\text{eff}}}{m} }\right)\,.
\end{align*}
We thus complete the proof of \cref{thm:linoto-reg}.
\end{proof}

\section{Offline-Reference Regret Analysis of \alg}
We start with the following lemma.
\begin{lemma}\label{lem:lower_bound_on_baseline}
    On the event $\cE$, it holds that $r_b \ge \muoff - (\alpha+2) S$.
\end{lemma}
\begin{proof}
    We have
    \begin{align*}
        r_b = \max_{a \in \cA} \inf_{\theta \in \cC_1} \ip{\theta, a} - \alpha S &\ge \frac1m \sum_{a \in \cA} m_a \inf_{\theta \in \cC_1} \ip{\theta, a} - \alpha S \\
        &= \frac1m \sum_{a \in \cA} m_a \left( \inf_{\theta \in \cC_1} -\ip{\theta_* - \theta, a} + \ip{\theta_*, a} \right) - \alpha S \\
        &= \frac1m \sum_{a \in \cA} -m_a \sup_{\theta \in \cC_1} \ip{\theta_* - \theta, a} + \frac1m \sum_{a \in \cA} m_a \ip{\theta_*, a} - \alpha S \\
        &\ge -\frac1m \sum_{a \in \cA} m_a \sup_{\theta \in \cC_1} \norm{\theta - \theta_*}_{V_0} \norm{a}_{V_0^{-1}} + \muoff - \alpha S \\
        &\ge -\frac{2\sqrt{\betaoff}}{m} \sum_{a \in \cA} m_a \norm{a}_{V_0^{-1}} + \muoff - \alpha S \\
        &\ge - 2S + \muoff - \alpha S \\
        &= \muoff - (\alpha+2) S \,.
    \end{align*}
This establishes Lemma \ref{lem:lower_bound_on_baseline}.
\end{proof}

\subsection{Proof of \cref{thm:linoto-offreg}}
\begin{proof}
Let $\cU_t \coloneqq \{s \in [t] \,:\, A_s = U_s\}$ and $\cL_t \coloneqq \{s \in [t] \,:\, A_s = L_s\}$. Then
\begin{align*}
    R_n^{\text{off}}(\alg) &= \sum_{t=1}^n \left( \muoff - \ip{\theta_*, A_t}\right) \\
    &= \sum_{t \in \cU_n} \left( \muoff - \ip{\theta_*, A_t}\right) + \sum_{t \in \cL_n} \left( \muoff - \ip{\theta_*, A_t}\right) \\
    &= \sum_{t \in \cU_n} \left( \muoff - \ip{\theta_*, U_t}\right) + \sum_{t \in \cL_n} \left( \muoff - \ip{\theta_*, L_t}\right) \,.
\end{align*}
To bound the regret corresponding to the rounds when the UCB action was chosen, consider the ``realized'' budget at the end of round $n$, that is, the excess reward of the actions chosen by \alg till round $n$ over the baseline $r_b$:
\begin{align*}
    B_+(n) &\coloneqq \sum_{t=1}^n \lcb_n(A_t) - nr_b \\
    &= \sum_{t=1}^{n-1} \lcb_n(A_t) + \lcb_n(A_n) - nr_b.
\end{align*}
If $A_n = U_n$, then $B_+(n) = B(n) \ge 0$. If $A_n = L_n$, then
\begin{align*}
    B_+(n) &= \sum_{t=1}^{n-1} \lcb_n(A_t) + \lcb_n(L_n) - nr_b \\
    &\ge \sum_{t=1}^{n-1} \lcb_{n-1}(A_t) - (n-1)r_b + \lcb_n(L_n) - r_b \\
    &\ge B_+(n-1) + \lcb_n(\lcbaoff) - r_b \tag{since $L_n = \argmax_{a \in \cA} \lcb_n(a)$} \\
    &\ge B_+(n-1) + \lcb_1(\lcbaoff) - r_b \tag{since LCB values are non-decreasing} \\
    &= B_+(n-1) + \alpha S \,.
\end{align*}
Similarly, $B_+(n-1) \ge 0$ or $B_+(n-1) \ge B_+(n-2) + \alpha S$. Therefore, by iterating the above argument, we have $B_+(n) \ge 0$ or $B_+(n) \ge |\cL_n| \alpha S$. In either case, we have $B_+(n) \ge 0$. This implies
\begin{align*}
     \sum_{t \in \cU_n} \lcb_n(A_t) + \sum_{t \in \cL_n} \lcb_n(A_t) - nr_b \ge 0 \,.
\end{align*}
Rearranging and using the fact that $A_t = U_t$ if $t \in \cU_n$ and $A_t = L_t$ if $t \in \cL_n$, we get
\begin{align*}
    \sum_{t \in \cU_n} \lcb_n(U_t) &\ge \sum_{t \in \cL_n} \left( r_b - \lcb_n(L_t) \right) + \sum_{t \in \cU_n} r_b \\
    &\ge \sum_{t \in \cL_n} \left( \muoff - \lcb_n(L_t) - (\alpha+2)S \right) + \sum_{t \in \cU_n} r_b \\
    &\ge \sum_{t \in \cL_n} \left( \muoff - \ip{\theta_*, L_t} \right) + \sum_{t \in \cU_n} r_b - (\alpha+2)Sn \,.
\end{align*}
Using the previous display with the regret during UCB iterations, we have
\begin{align*}
    \sum_{t \in \cU_n} \left( \muoff - \ip{\theta_*, U_t}\right) &\le \sum_{t \in \cU_n} \left( \muoff - \lcb_n(U_t)\right) \\
    &\le \sum_{t \in \cU_n} \left( \muoff - r_b \right) + (\alpha+2)Sn - \sum_{t \in \cL_n} \left( \muoff - \ip{\theta_*, L_t} \right) \\
    &\le 2(\alpha+2)Sn - \sum_{t \in \cL_n} \left( \muoff - \ip{\theta_*, L_t} \right). \tag{by Lemma~\ref{lem:lower_bound_on_baseline}}
\end{align*}
Finally, the total $R_n^{\text{off}}(\alg)$ can be bounded as
\begin{align*}
    R_n^{\text{off}}(\alg) &= \sum_{t \in \cU_n} \left( \muoff - \ip{\theta_*, U_t}\right) + \sum_{t \in \cL_n} \left( \muoff - \ip{\theta_*, L_t}\right) \\
    &\le 2(\alpha+2)Sn \\
    &= 2(\alpha+2) \sqrt{\frac{d\betaoff}{m}} n \,.
\end{align*}
This completes the proof of \cref{thm:linoto-offreg}.
\end{proof}

\subsection{Proof of \Cref{prop:lcb-offreg}}
The proof follows a similar structure to the proof of \cref{prop:loff-subopt}.
\begin{proof}
When the event $\theta_* \in \cC_t$ for all $t \in [n]$ holds, we have that with probability $1 - \delta$ and for all $t \in [n]$,
\begin{align*}
    \mu_{L_t}= \ip{\theta_* , L_t} \ge \max_{t' \le t}\inf_{\theta \in \cC_{t'}} \ip{\theta, L_t} = \lcb_t(L_t) \,.
\end{align*}
Since $\max$ is always greater than the average, we have
\begin{align*}
    \lcb_t(L_t) = \max_{a \in \cA} \lcb_t(a) \ge \frac1m \sum_{a \in \cA} m_a \lcb_t(a) = \frac1m \sum_{a \in \cA} m_a \max_{t' \le t} \inf_{\theta \in \cC_{t'}} \ip{\theta, a} \,.
\end{align*} 
From the above two displays, we have
\begin{align*}
    \mu_{L_t} &\ge \frac1m \sum_{a \in \cA} m_a \max_{t' \le t} \inf_{\theta \in \cC_{t'}} \ip{\theta, a} \\
     &= \frac1m \sum_{a \in \cA} m_a \max_{t' \le t} \inf_{\theta \in \cC_{t'}} \left( \ip{\theta_*, a} - \ip{\theta_* - \theta, a} \right) \\
     &= \frac1m \sum_{a \in \cA} m_a \mu_a + \frac1m \sum_{a \in \cA} m_a \max_{t' \le t} \left(-\sup_{\theta \in \cC_{t'}} \ip{\theta_* - \theta, a} \right) \\
     &\ge \muoff + \frac1m \sum_{a \in \cA} m_a \max_{t' \le t}\left( -\sup_{\theta \in \cC_{t'}} \norm{\theta - \theta_*}_{V_{t'-1}} \norm{a}_{V_{t'-1}^{-1}} \right) \tag{Cauchy-Schwarz} \\
     &\ge \muoff + \frac1m \sum_{a \in \cA} m_a \max_{t' \le t}\left( - 2\sqrt{\rho_{t'}} \norm{a}_{V_{t'-1}^{-1}} \right)\tag{since $\sup_{\theta \in \cC_{t'}} \norm{\theta - \theta_*}_{V_{t'-1}} \le 2\sqrt{\rho_{t'}}$} \\
     &> \muoff - \frac{2 \sqrt{\rho_t}}{m} \sum_{a \in \cA} m_a \max_{t' \le t} \norm{a}_{V_{t'-1}^{-1}}  \tag{since $\rho_t > \rho_{t'}$ for all $t' \le t$}\\
     &= \muoff - \frac{2 \sqrt{\rho_t}}{m} \sum_{a \in \cA} m_a \norm{a}_{V_0^{-1}} \tag{since $V_{t'-1}^{-1} \preceq V_0^{-1}$ for all $t'$} \\
     &\ge \muoff - 2 \sqrt{\frac{d \rho_t}{m}} \tag{using \cref{lem:norm_bound}} \,.
\end{align*}
Rearranging, the $R_n^{\text{off}}(\lcb)$ is bounded as
\begin{align*}
    R_n^{\text{off}}(\lcb, \nu) &= \sum_{t=1}^n \left( \frac1m \sum_{a \in \cA} m_a \mu_a - \mu_{L_t} \right) \le  2\sqrt{\frac{d}{m}} \sum_{t=1}^n \sqrt{\rho_t} \le 2 n \sqrt{\frac{d \rho_n}{m}}. 
\end{align*}
This establishes Proposition \ref{prop:lcb-offreg}.
\end{proof}

\section{Additional Experimental Results}\label{app:more-expts}

\Cref{fig:d} shows the regret of LinUCB, LinLCB, \alg, and a random policy by varying the number of dimensions $d \in \{10, 20, 50\}$. The other problem parameters are set to the default values as described in \Cref{sec:experiments}. In \Cref{fig:alpha-sensitivity}, we plot the sensitivity of \alg with respect to the tuning parameter $\alpha$ (see \cref{eq:baseline_reward}) for different number of offline samples $m$.

\begin{figure}[!t]
    \centering
    \begin{subfigure}{\textwidth}
        \centering
    \includegraphics[width=\linewidth]{legend_only.pdf}
    \end{subfigure}

    \centering
    \begin{subfigure}{\textwidth}
        \centering
    \includegraphics[width=\linewidth]{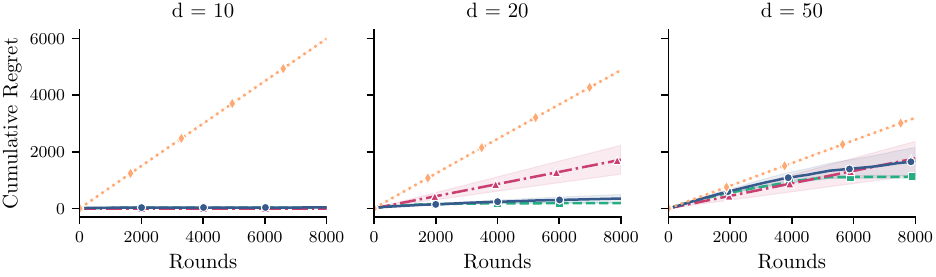}
    \end{subfigure}
    \caption{Cumulative regret with varying dimension $d$.}
    \label{fig:d}
\end{figure}

\begin{figure}[!t]
        \centering
    \includegraphics[width=\linewidth]{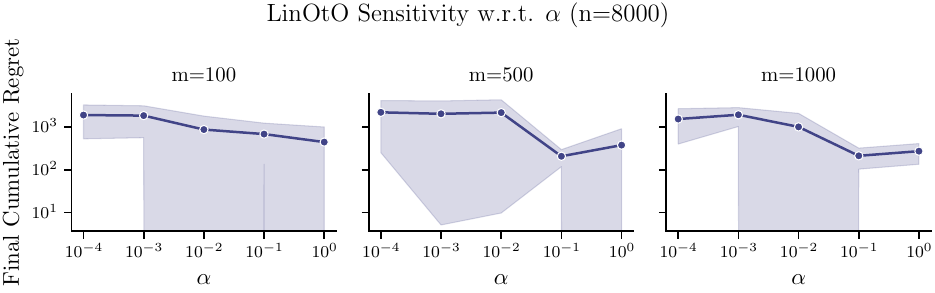}
    \caption{Sensitivity of LinOtO with respect to $\alpha$ for varying number of offline samples $m$. The vertical axis is on log scale.}
    \label{fig:alpha-sensitivity}
\end{figure}

\typeout{get arXiv to do 4 passes: Label(s) may have changed. Rerun}

\end{document}